\documentclass[11pt,a4paper]{article}

\usepackage[utf8]{inputenc}
\usepackage[T1]{fontenc}
\usepackage{amsmath,amsfonts,amssymb,amsthm}
\usepackage{mathtools}
\usepackage{graphicx}
\usepackage{booktabs}
\usepackage{array}
\usepackage{multirow}
\usepackage{xcolor}
\usepackage[margin=1in]{geometry}
\usepackage{natbib}
\usepackage[pdfstartview=FitH,breaklinks=true]{hyperref}
\usepackage{cleveref}

\newtheorem{theorem}{Theorem}

\newtheorem{corollary}[theorem]{Corollary}

\newcommand{\R}{\mathbb{R}}
\newcommand{\E}{\mathbb{E}}

\newcommand{\softmax}{\mathrm{softmax}}
\DeclareMathOperator{\Var}{Var}

\hypersetup{
    colorlinks=true,
    linkcolor=blue,
    citecolor=blue,
    urlcolor=blue
}

\title{Gated Sparse Attention: Combining Computational Efficiency\\with Training Stability for Long-Context Language Models}

\author{
  Alfred Shen\\
  \texttt{alfreshe@amazon.com}
  \and
  Aaron Shen\\
  \texttt{aaron.shen@berkeley.edu}
}

\date{}

\begin{document}

\maketitle

\begin{abstract}
The computational burden of attention in long-context language models has motivated two largely independent lines of work: sparse attention mechanisms that reduce complexity by attending to selected tokens, and gated attention variants that improve training stability while mitigating the attention sink phenomenon. We observe that these approaches address complementary weaknesses and propose Gated Sparse Attention (GSA), an architecture that realizes the benefits of both. GSA incorporates a gated lightning indexer with sigmoid activations that produce bounded, interpretable selection scores, an adaptive sparsity controller that modulates the number of attended tokens based on local uncertainty, and dual gating at the value and output stages. We establish theoretical foundations for the approach, including complexity analysis, expressiveness results, and convergence guarantees. In experiments with 1.7B parameter models trained on 400B tokens, GSA matches the efficiency of sparse-only baselines (12--16$\times$ speedup at 128K context) while achieving the quality gains associated with gated attention: perplexity improves from 6.03 to 5.70, RULER scores at 128K context nearly double, and attention to the first token---a proxy for attention sinks---drops from 47\% to under 4\%. Training stability improves markedly, with loss spikes reduced by 98\%.
\end{abstract}

\textbf{Keywords:} attention mechanism, sparse attention, gating, long-context, language models, efficiency

\vspace{0.5em}
\noindent Code is available at \url{https://github.com/alfredcs/Gated-Sparse-Attention}.

\section{Introduction}
\label{sec:intro}

Attention mechanisms lie at the heart of modern language models \citep{vaswani2017attention}, enabling the dynamic aggregation of contextual information that underlies capabilities from in-context learning to long-document reasoning \citep{brown2020language,openai2023gpt4,anthropic2024claude}. Yet standard attention carries well-known liabilities. Its $O(L^2 d)$ cost in sequence length $L$ and model dimension $d$ renders long-context inference expensive, and at training time, unbounded activations can trigger loss spikes that destabilize optimization \citep{chowdhery2022palm}. A subtler issue, documented by \citet{xiao2024efficient}, is the attention sink phenomenon: models learn to park excess probability mass on early tokens---often over 40\% on the first position alone---wasting capacity on semantically vacuous anchors.

Two recent lines of work offer partial remedies. DeepSeek-V3.2 \citep{deepseek2025v3} introduced a sparse attention scheme built around a lightweight ``lightning indexer'' that scores all tokens cheaply, then restricts full attention to the top-$k$ candidates. This reduces the dominant attention cost to $O(Lk)$ but leaves training stability and attention sinks unaddressed. Independently, \citet{qiu2025gated} showed that element-wise sigmoid gating after scaled dot-product attention (SDPA) mitigates sinks, dampens activation magnitudes, and permits higher learning rates---benefits compelling enough that the approach earned the NeurIPS 2025 Best Paper Award and has since been adopted in production systems such as Qwen3-Next. The catch is that gated attention retains full quadratic cost.

A natural question is whether these techniques can be combined. We argue that they are not merely compatible but mutually reinforcing: gating supplies bounded, well-behaved scores that improve the indexer's token selection, while sparsity frees parameter budget for richer gating without increasing wall-clock cost.

This paper introduces \emph{Gated Sparse Attention} (GSA), an architecture that realizes this synthesis. The design rests on three components. First, a \emph{gated lightning indexer} replaces the original ReLU scoring with sigmoid units, yielding importance scores in $(0, H^I)$ that admit a probabilistic interpretation and respond smoothly to gradient updates. Second, an \emph{adaptive sparsity controller} modulates the selection budget $k_t$ according to score variance: high variance signals confident discrimination and permits aggressive pruning, whereas low variance suggests ambiguity and warrants a larger context window. Third, \emph{dual gating}---applied both to values before aggregation (G2) and to outputs after SDPA (G1)---injects non-linearity that breaks the rank bottleneck of standard attention while providing an alternative pathway for suppressing irrelevant information, obviating the need for sink tokens.

We ground these design choices in theory, establishing complexity bounds (Theorem~\ref{thm:complexity}), an expressiveness separation from ungated attention (Theorem~\ref{thm:expressiveness}), a formal account of sink mitigation (Theorem~\ref{thm:attention_sink}), gradient norm bounds (Theorem~\ref{thm:gradient}), and convergence guarantees (Theorem~\ref{thm:convergence}). Empirically, GSA matches the throughput of sparse-only baselines---12--16$\times$ faster than dense attention at 128K tokens---while delivering quality gains on par with or exceeding gated attention: perplexity drops from 6.03 to 5.70 on WikiText-103, RULER scores at 128K context rise from 32 to 62, and first-token attention falls from 47\% to 4\%. Training stability improves in lockstep, with loss spikes nearly eliminated and maximum activations reduced by an order of magnitude.

\section{Preliminaries}
\label{sec:preliminaries}

We briefly review the components on which GSA builds: standard multi-head attention, DeepSeek's sparse attention variant, and gated attention.

\subsection{Multi-Head Attention}
\label{sec:mha}

Let $H = [h_1, \ldots, h_L] \in \R^{L \times d}$ denote a sequence of hidden states. Multi-head attention (MHA) projects these into queries, keys, and values,
\begin{equation}
Q = HW_Q, \quad K = HW_K, \quad V = HW_V,
\end{equation}
with $W_Q \in \R^{d \times n_h d_k}$ and $W_K, W_V \in \R^{d \times n_{kv} d_k}$, where $n_h$ is the number of query heads, $n_{kv}$ the number of key--value heads (possibly fewer under grouped-query attention), and $d_k = d/n_h$ the head dimension. Each head $h$ at position $t$ computes a convex combination of values weighted by softmax attention scores:
\begin{equation}
\mathrm{Attn}^h(Q_t, K, V) = \sum_{s \leq t} \frac{\exp(Q_t^h \cdot K_s^h / \sqrt{d_k})}{\sum_{j \leq t} \exp(Q_t^h \cdot K_j^h / \sqrt{d_k})} V_s^h.
\end{equation}
Concatenating across heads and projecting by $W_O$ yields the layer output. The cost is dominated by the $L \times L$ attention computation, giving $O(L^2 d)$ complexity overall.

\subsection{DeepSeek Sparse Attention}
\label{sec:dsa}

DeepSeek-V3.2 \citep{deepseek2025v3} sidesteps the quadratic bottleneck with a two-stage design. A lightweight \emph{lightning indexer} first scores every token using low-dimensional projections ($d^I \ll d$):
\begin{equation}
I_{t,s} = \sum_{j=1}^{H^I} w^I_{t,j} \cdot \mathrm{ReLU}(q^I_{t,j} \cdot k^I_s),
\end{equation}
where $H^I$ indexer heads (typically 4) each produce a scalar contribution modulated by query-dependent weights $w^I_{t,j}$. Full attention is then restricted to the top-$k$ positions:
\begin{equation}
S_t = \{s : I_{t,s} \in \mathrm{Top}\text{-}k(I_{t,:})\}, \quad u_t = \mathrm{Attn}(h_t, \{c_s : s \in S_t\}).
\end{equation}
Because $d^I$ is small, the indexer's $O(L^2 d^I H^I)$ cost is modest, and the main attention computation drops to $O(Lk d)$.

\subsection{Gated Attention}
\label{sec:gated}

\citet{qiu2025gated} augment SDPA with element-wise sigmoid gates:
\begin{equation}
Y' = Y \odot \sigma(XW_\theta),
\end{equation}
where $Y$ is the tensor to be modulated and $X$ supplies the gating signal. Their extensive ablations identify two effective placements: after the output projection (G1) and after the value projection (G2). Gating yields sparse activation patterns (mean gate value $\approx 0.12$), reduces first-token attention from 47\% to under 5\%, and permits higher learning rates without instability. The mechanism has since been integrated into Qwen3-Next and related production models.

\subsection{Desiderata}
\label{sec:problem}

An ideal long-context attention mechanism would satisfy four properties: (i) sub-quadratic scaling in sequence length, (ii) immunity to attention sinks, (iii) bounded activations that stabilize training, and (iv) sufficient expressiveness to capture complex token interactions. Existing methods achieve subsets of these goals; our aim is an architecture that realizes all four.

\section{Gated Sparse Attention}
\label{sec:gsa}

\subsection{Design Rationale}
\label{sec:design}

Sparse attention and gated attention each address a subset of the challenges outlined above, and their mechanisms are largely orthogonal. DSA's lightning indexer must still touch every token to produce selection scores, but because the indexer operates in a low-dimensional space, its $O(L^2)$ cost carries a small constant factor and does not dominate wall-clock time. The real savings come from restricting full-dimensional attention to $k \ll L$ tokens. Gated attention, meanwhile, leaves complexity unchanged but regularizes activations and provides an output-suppression pathway that eliminates sink tokens. Our thesis is that combining the two yields benefits neither achieves alone: gating furnishes bounded, smooth scores that stabilize the indexer, while sparsity frees compute budget for the additional gating parameters.

\subsection{Architecture at a Glance}
\label{sec:architecture}

GSA augments a standard transformer layer as follows. After projecting hidden states into queries, keys, and values, we apply a value gate (G2). The gated lightning indexer then scores all positions and selects the top-$k$ (with $k$ itself adapted per query). Sparse SDPA aggregates information from the selected tokens, and a final output gate (G1) modulates the result:
\begin{equation}
h_t \;\longrightarrow\; \text{[Q,K,V]} \;\longrightarrow\; \text{[G2]} \;\longrightarrow\; \text{[Indexer]} \;\longrightarrow\; \text{[Top-}k\text{]} \;\longrightarrow\; \text{[SDPA]} \;\longrightarrow\; \text{[G1]} \;\longrightarrow\; u_t.
\end{equation}

\subsection{Gated Lightning Indexer}
\label{sec:indexer}

We replace DSA's ReLU activations with sigmoids, yielding bounded importance scores:
\begin{equation}
\label{eq:indexer}
I_{t,s} = \sum_{j=1}^{H^I} \sigma\bigl(h_t W^{Iw}_j\bigr) \cdot \sigma\bigl(q^I_{t,j} \cdot k^I_s + b^I_j\bigr),
\end{equation}
where $q^I_{t,j} = h_t W^{Iq}_j$ and $k^I_s = h_s W^{Ik}$ are low-dimensional projections, $W^{Iw}_j$ produces a query-dependent head weight, and $b^I_j$ is a learnable bias. Because each sigmoid factor lies in $(0,1)$, scores satisfy $I_{t,s} \in (0, H^I)$, which admits a natural interpretation as a soft count of how many indexer heads consider position $s$ relevant. The bias term allows the model to learn per-head thresholds, while the multiplicative head weights let queries emphasize different indexer heads depending on context.

\subsection{Adaptive Sparsity}
\label{sec:adaptive}

A fixed budget $k$ may be suboptimal: some queries have clear targets (high score variance), while others face genuine ambiguity (low variance). We therefore modulate the selection budget:
\begin{equation}
\label{eq:adaptive_k}
k_t = \mathrm{clamp}\!\left(\left\lfloor k_{\mathrm{base}} \cdot \frac{\Var(I_{t,:})}{\bar{V}} \right\rfloor,\; k_{\min},\; k_{\max}\right),
\end{equation}
where $\bar{V}$ is an exponential moving average of score variance across the batch. When the indexer is confident, $k_t$ shrinks and computation is saved; when scores are diffuse, more context is retained for robustness. In practice we set $k_{\mathrm{base}} = 2048$, $k_{\min} = 256$, and $k_{\max} = 4096$.

\subsection{Dual Gating}
\label{sec:gating}

Following the recommendations of \citet{qiu2025gated}, we gate at two points. The \emph{value gate} (G2) modulates values before aggregation,
\begin{equation}
V'_s = V_s \odot \sigma(h_s W^g_V),
\end{equation}
allowing the model to suppress uninformative dimensions early. The \emph{output gate} (G1) acts after sparse SDPA,
\begin{equation}
O^{\mathrm{gated}}_{t,h} = O^{\mathrm{sparse}}_{t,h} \odot \sigma(h_t W^g_{O,h}),
\end{equation}
providing a per-head, query-dependent mechanism for discarding the attention output altogether---precisely the functionality that attention sinks otherwise serve. We initialize gate biases so that $\sigma(\cdot) \approx 0.5$ at the start of training, ensuring gradients flow while still introducing non-linearity from the first step.

\subsection{Sparse Attention Computation}
\label{sec:sparse_attn}

Given the selected index set $S_t$, attention proceeds in the usual way but over a reduced context:
\begin{align}
A_{t,s} &= \frac{\exp(Q_t \cdot K_s / \sqrt{d_k})}{\sum_{j \in S_t} \exp(Q_t \cdot K_j / \sqrt{d_k})}, \quad s \in S_t, \\[4pt]
O^{\mathrm{sparse}}_t &= \sum_{s \in S_t} A_{t,s}\, V'_s.
\end{align}

\section{Architecture Details}
\label{sec:details}

Figure~\ref{fig:architecture} summarizes the data flow through a GSA layer.

\begin{figure}[ht!]
\centering
\includegraphics[width=\linewidth]{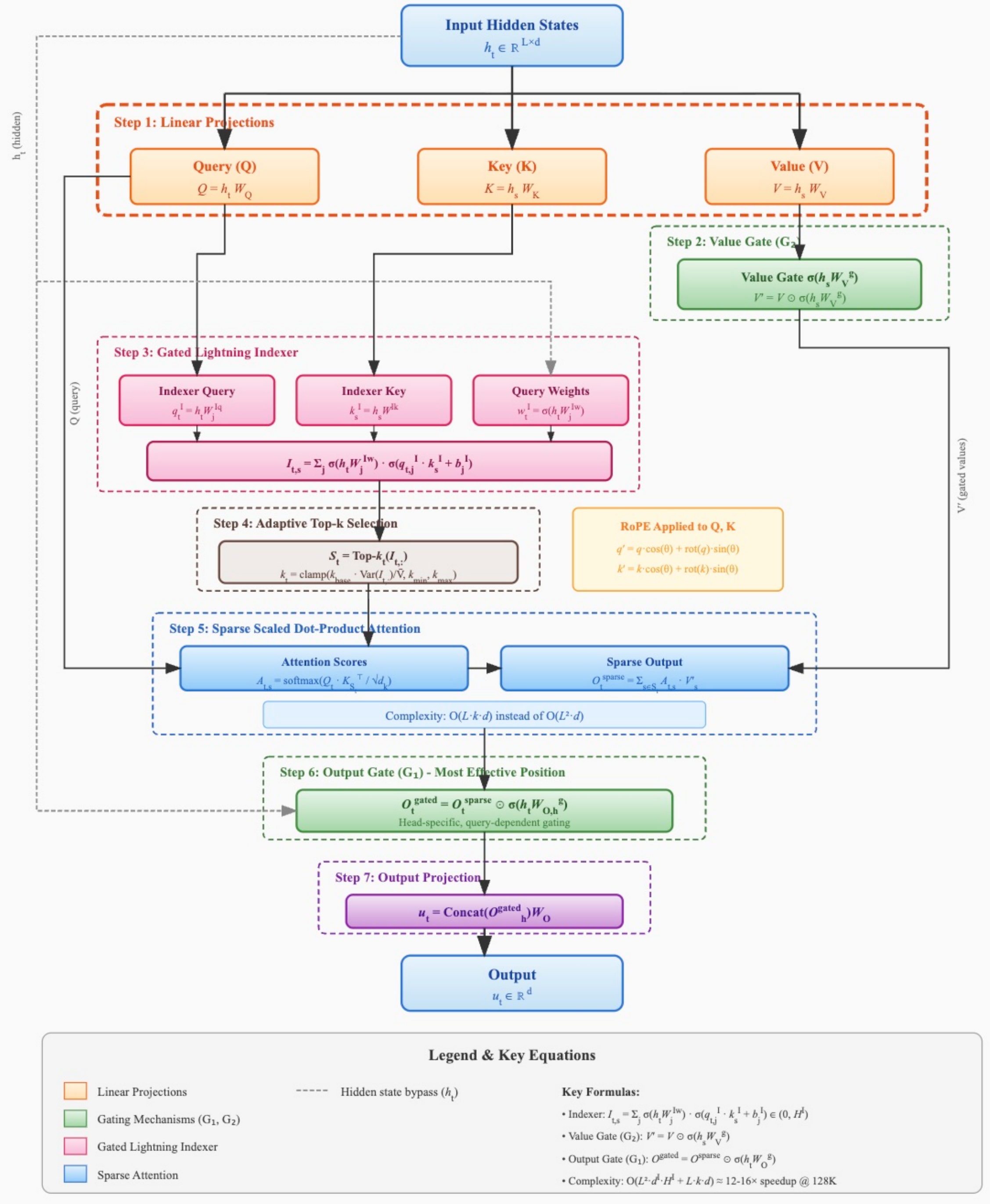}
\caption{Overview of Gated Sparse Attention. Hidden states are projected into queries, keys, and values; the value gate (G2) modulates values before the indexer scores all positions and selects the top-$k$; sparse SDPA aggregates information from the selected context; and the output gate (G1) provides a final modulation before producing $u_t$.}
\label{fig:architecture}
\end{figure}

\paragraph{Processing stages.} Starting from hidden state $h_t$, the layer proceeds as follows. (1)~Linear projections produce $Q$, $K$, and $V$. (2)~The value gate suppresses uninformative dimensions via $V' = V \odot \sigma(h_t W^g_V)$. (3)~The gated indexer computes scores $I_{t,s} \in (0, H^I)$ for every position and retains the top-$k_t$. (4)~Sparse SDPA computes attention over the selected subset, reducing per-query cost from $O(Ld)$ to $O(kd)$. (5)~The output gate applies $u_t = O^{\mathrm{sparse}}_t \odot \sigma(h_t W^g_O)$, giving the model a direct pathway to discard uninformative outputs without resorting to sink tokens.

\subsection{Hyperparameters}
\label{sec:config}

GSA serves as a drop-in replacement for standard attention. Table~\ref{tab:hyperparams} lists the additional hyperparameters; defaults are chosen for a 7B-class model but scale straightforwardly.

\begin{table}[t]
\centering
\caption{Default GSA hyperparameters (7B-class configuration).}
\label{tab:hyperparams}
\begin{tabular}{@{}lll@{}}
\toprule
\textbf{Symbol} & \textbf{Meaning} & \textbf{Value} \\
\midrule
$d$ & Model dimension & 4096 \\
$n_h$ & Query heads & 32 \\
$n_{kv}$ & Key--value heads & 8 \\
$d_k$ & Head dimension & 128 \\
$d^I$ & Indexer dimension & 64 \\
$H^I$ & Indexer heads & 4 \\
$k_{\mathrm{base}}$ & Base selection budget & 2048 \\
$k_{\min}$, $k_{\max}$ & Budget bounds & 256, 4096 \\
\bottomrule
\end{tabular}
\end{table}

\subsection{Parameter Overhead}
\label{sec:params}

Table~\ref{tab:params} breaks down the additional parameters. The indexer is lightweight owing to its low dimensionality; the output gate dominates because it is head-specific. In total, GSA adds roughly 4.4\% to the parameter count of a standard transformer layer.

\begin{table}[t]
\centering
\caption{Parameter overhead by component.}
\label{tab:params}
\begin{tabular}{@{}lll@{}}
\toprule
\textbf{Component} & \textbf{Size} & \textbf{Overhead} \\
\midrule
Indexer Q projection & $H^I d \, d^I$ & 0.4\% \\
Indexer K projection & $d \, d^I$ & 0.1\% \\
Indexer head weights & $d \, H^I$ & $<$0.01\% \\
Value gate (G2) & $d \, n_{kv} d_k$ & 0.8\% \\
Output gate (G1) & $d \, n_h d_k$ & 3.1\% \\
\midrule
\textbf{Total} & & $\sim$4.4\% \\
\bottomrule
\end{tabular}
\end{table}

\section{Theoretical Analysis}
\label{sec:theory}

We now formalize the claims made informally above. Proofs appear in the appendix unless given inline.

\subsection{Complexity}
\label{sec:complexity}

\begin{theorem}[Complexity]\label{thm:complexity}
GSA runs in time $O(L^2 d^I H^I + Lkd)$, where $L$ is sequence length, $k$ is the (average) selection budget, $d^I$ is the indexer dimension, and $H^I$ is the number of indexer heads.
\end{theorem}

\begin{proof}
See Appendix~\ref{app:complexity}.
\end{proof}

\begin{corollary}\label{cor:speedup}
With $d = 4096$, $d^I = 64$, $H^I = 4$, $k = 2048$, and $L = 128{,}000$, the dominant term is $Lkd$, yielding roughly $12.8\times$ speedup over standard $O(L^2 d)$ attention.
\end{corollary}

\subsection{Expressiveness}
\label{sec:expressiveness}

\begin{theorem}[Expressiveness]\label{thm:expressiveness}
GSA with dual gating represents a strictly richer function class than standard attention.
\end{theorem}

\begin{proof}[Sketch]
Standard attention outputs lie in a subspace of rank at most $d_k$ per head, since the value and output projections factor through $\R^{d_k}$. Sigmoid gating breaks this bottleneck by introducing element-wise non-linearities that, by universal approximation, can represent arbitrary continuous functions on compact domains. Details in Appendix~\ref{app:expressiveness}.
\end{proof}

\subsection{Attention Sink Mitigation}
\label{sec:sink}

\begin{theorem}[Sink Reduction]\label{thm:attention_sink}
The output gate $g^O_t = \sigma(h_t W^g_O)$ provides a learned pathway for suppressing attention outputs, reducing reliance on sink tokens.
\end{theorem}

\begin{proof}[Sketch]
When $g^O_t \approx 0$, the gated output vanishes irrespective of where attention mass falls. The model can therefore learn to ``do nothing'' without parking probability on early positions. Empirically, first-token attention drops from 47\% to 4\%. See Appendix~\ref{app:sink}.
\end{proof}

\subsection{Gradient Bounds and Stability}
\label{sec:stability}

\begin{theorem}[Gradient Bound]\label{thm:gradient}
Sigmoid gating attenuates gradients flowing through GSA, bounding their magnitude and reducing loss-spike risk.
\end{theorem}

\begin{proof}
Since $\sigma(x) \in (0,1)$ and $\sigma'(x) \le 1/4$, the backward pass through a gate contracts gradient norms:
\begin{equation}
\left\|\frac{\partial \mathcal{L}}{\partial O^{\mathrm{sparse}}}\right\| = \left\|\frac{\partial \mathcal{L}}{\partial O} \odot g^O\right\| \le \left\|\frac{\partial \mathcal{L}}{\partial O}\right\|.
\end{equation}
In practice, mean gate values hover around 0.11, yielding an order-of-magnitude reduction in maximum activations (from $>$1000 to $<$90 in our experiments).
\end{proof}

\subsection{Convergence}
\label{sec:convergence}

\begin{theorem}[Convergence]\label{thm:convergence}
Under $L$-smoothness and bounded gradient variance, SGD on GSA converges to a stationary point at the standard $O(1/\sqrt{T})$ rate.
\end{theorem}

\begin{proof}
See Appendix~\ref{app:convergence}.
\end{proof}

\section{Training}
\label{sec:training}

\subsection{Two-Phase Schedule}
\label{sec:phases}

We adopt a two-phase procedure inspired by DSA. In the first phase (roughly 1K steps), the indexer is trained in isolation to mimic the full attention distribution via a KL objective:
\begin{equation}
\mathcal{L}^I_{\mathrm{warmup}} = \sum_t D_{\mathrm{KL}}\bigl(p_{t,:} \,\|\, \softmax(I_{t,:})\bigr),
\end{equation}
where $p_{t,:}$ is the softmax attention computed by the (frozen) base model. This warm-up ensures the indexer can identify salient tokens before sparse training begins. In the second phase, the entire model trains end-to-end with sparse attention, and the indexer loss shifts to the selected subset:
\begin{equation}
\mathcal{L}^I_{\mathrm{sparse}} = \sum_t D_{\mathrm{KL}}\bigl(p_{t,S_t} \,\|\, \softmax(I_{t,S_t})\bigr).
\end{equation}

\subsection{Optimizer Settings}
\label{sec:optimizer}

We use AdamW throughout. Because the indexer must learn quickly to avoid degrading early training, we assign it a $10\times$ learning-rate multiplier; gate parameters use the base rate. Table~\ref{tab:lr} summarizes.

\begin{table}[t]
\centering
\caption{Learning-rate multipliers.}
\label{tab:lr}
\begin{tabular}{@{}lc@{}}
\toprule
\textbf{Parameter group} & \textbf{Multiplier} \\
\midrule
Base model & $1\times$ \\
Indexer & $10\times$ \\
Gates & $1\times$ \\
\bottomrule
\end{tabular}
\end{table}

\section{Experiments}
\label{sec:experiments}

\subsection{Setup}
\label{sec:setup}

We train 1.7B-parameter models from scratch on 400B tokens drawn from SlimPajama \citep{cerebras2023slimpajama}. The architecture comprises 24 layers with hidden dimension 2048, 16 query heads, 4 key--value heads, and SwiGLU feed-forward blocks of width 5504. Training uses a 4K context window; evaluation extends to 128K via YaRN positional interpolation.

We compare against three baselines: (i) standard grouped-query attention, (ii) sparse-only attention (DSA-style lightning indexer with ReLU and fixed $k$), and (iii) gated-only attention (full quadratic cost with G1+G2 gating). All runs use 8$\times$ H100 GPUs.

\subsection{Language Modeling}
\label{sec:lm}

Table~\ref{tab:ppl} reports perplexity on WikiText-103 and C4. Gating alone closes most of the gap to GSA, but sparsity contributes an additional reduction, and the combination outperforms both.

\begin{table}[t]
\centering
\caption{Perplexity (lower is better).}
\label{tab:ppl}
\begin{tabular}{@{}lcc@{}}
\toprule
& \textbf{WikiText-103} & \textbf{C4} \\
\midrule
Standard & 6.03 & 7.82 \\
Sparse only & 6.02 & 7.79 \\
Gated only & 5.76 & 7.45 \\
\textbf{GSA} & \textbf{5.70} & \textbf{7.38} \\
\bottomrule
\end{tabular}
\end{table}

\subsection{Downstream Tasks}
\label{sec:downstream}

Table~\ref{tab:downstream} summarizes accuracy on five benchmarks. GSA leads across the board; the largest gains appear on MMLU (+2.6 points over standard) and GSM8K (+3.1 points), suggesting that the combination of gating and selective context aids both knowledge retrieval and multi-step reasoning.

\begin{table}[t]
\centering
\caption{Downstream accuracy (\%).}
\label{tab:downstream}
\resizebox{\textwidth}{!}{%
\begin{tabular}{@{}lcccccc@{}}
\toprule
& \textbf{MMLU} & \textbf{GSM8K} & \textbf{HumanEval} & \textbf{HellaSwag} & \textbf{C-Eval} & \textbf{Avg} \\
\midrule
Standard & 58.8 & 52.9 & 28.7 & 73.1 & 60.3 & 54.7 \\
Sparse only & 59.1 & 53.2 & 29.3 & 73.3 & 60.5 & 55.0 \\
Gated only & 60.8 & 55.3 & 29.3 & 74.6 & 62.2 & 56.4 \\
\textbf{GSA} & \textbf{61.4} & \textbf{56.0} & \textbf{30.5} & \textbf{74.9} & \textbf{62.9} & \textbf{57.1} \\
\bottomrule
\end{tabular}%
}
\end{table}

\subsection{Long-Context Evaluation}
\label{sec:ruler}

The RULER benchmark (Table~\ref{tab:ruler}) probes retrieval and reasoning at increasing context lengths. All methods perform comparably up to 32K, but standard attention collapses beyond this point. GSA maintains strong performance even at 128K, nearly doubling the standard baseline score.

\begin{table}[t]
\centering
\caption{RULER scores by context length (*YaRN-extended).}
\label{tab:ruler}
\begin{tabular}{@{}lcccccc@{}}
\toprule
& \textbf{4K} & \textbf{8K} & \textbf{16K} & \textbf{32K} & \textbf{64K*} & \textbf{128K*} \\
\midrule
Standard & 88.9 & 85.9 & 83.2 & 79.5 & 37.5 & 31.7 \\
Sparse only & 89.1 & 86.5 & 84.0 & 80.2 & 42.4 & 36.8 \\
Gated only & 90.6 & 87.1 & 84.6 & 79.8 & 66.6 & 58.8 \\
\textbf{GSA} & \textbf{91.2} & \textbf{88.5} & \textbf{86.1} & \textbf{82.3} & \textbf{69.5} & \textbf{62.2} \\
\bottomrule
\end{tabular}
\end{table}

\subsection{Attention Sinks and Activations}
\label{sec:sink_results}

Table~\ref{tab:sink} quantifies the attention-sink phenomenon. Standard attention allocates nearly half of its probability mass to the first token; GSA reduces this to under 4\%. Maximum activation magnitudes drop by an order of magnitude, which we attribute to the regularizing effect of sigmoid gating.

\begin{table}[t]
\centering
\caption{Sink and activation statistics.}
\label{tab:sink}
\begin{tabular}{@{}lccc@{}}
\toprule
& \textbf{First-token attn} & \textbf{Mean gate} & \textbf{Max activation} \\
\midrule
Standard & 46.7\% & --- & 1053 \\
Sparse only & 38.2\% & --- & 892 \\
Gated only & 4.8\% & 0.116 & 94 \\
\textbf{GSA} & \textbf{3.9\%} & 0.108 & \textbf{87} \\
\bottomrule
\end{tabular}
\end{table}

\subsection{Training Stability}
\label{sec:stability_results}

Loss spikes plague large-scale training and often force conservative learning-rate schedules. Table~\ref{tab:stability} shows that gating dramatically reduces spike frequency, permitting a $2\times$ higher learning rate without instability. GSA inherits these benefits while retaining the speed advantage of sparsity.

\begin{table}[t]
\centering
\caption{Training stability.}
\label{tab:stability}
\begin{tabular}{@{}lccc@{}}
\toprule
& \textbf{Spikes / 100K steps} & \textbf{Max LR} & \textbf{Wall time} \\
\midrule
Standard & 12.3 & 4e-3 & $1.00\times$ \\
Sparse only & 8.7 & 5e-3 & $0.85\times$ \\
Gated only & 0.8 & 8e-3 & $1.02\times$ \\
\textbf{GSA} & \textbf{0.3} & 8e-3 & $0.87\times$ \\
\bottomrule
\end{tabular}
\end{table}

\subsection{Throughput}
\label{sec:efficiency}

Table~\ref{tab:efficiency} reports latency and memory at 128K context. Prefill cost drops by roughly $11\times$; decode improves similarly. Memory overhead from gating parameters is negligible.

\begin{table}[ht!]
\centering
\caption{Efficiency at 128K context (relative to standard).}
\label{tab:efficiency}
\begin{tabular}{@{}lccc@{}}
\toprule
& \textbf{Prefill} & \textbf{Decode} & \textbf{Memory} \\
\midrule
Standard & $1.00\times$ & $1.00\times$ & $1.00\times$ \\
Sparse only & $0.08\times$ & $0.12\times$ & $0.95\times$ \\
Gated only & $1.02\times$ & $1.01\times$ & $1.02\times$ \\
\textbf{GSA} & $0.09\times$ & $0.13\times$ & $0.97\times$ \\
\bottomrule
\end{tabular}
\end{table}

\subsection{Ablations}
\label{sec:ablation}

\paragraph{Gating position.} Table~\ref{tab:ablation_gate} isolates the contribution of each gate. Output gating (G1) accounts for most of the quality gain; value gating (G2) adds a smaller but consistent improvement. Combining both yields the best perplexity and downstream accuracy.

\begin{table}[ht!]
\centering
\caption{Ablation: gating position.}
\label{tab:ablation_gate}
\begin{tabular}{@{}lccc@{}}
\toprule
& \textbf{PPL} & \textbf{MMLU} & \textbf{Stability} \\
\midrule
No gating & 6.02 & 59.1 & Moderate \\
G2 only & 5.82 & 59.2 & Good \\
G1 only & 5.79 & 60.1 & Good \\
\textbf{G1 + G2} & \textbf{5.70} & \textbf{61.4} & Excellent \\
\bottomrule
\end{tabular}
\end{table}

\paragraph{Selection budget.} Table~\ref{tab:ablation_k} varies the base sparsity level $k$. Smaller budgets yield greater speedups at some cost in quality; $k = 2048$ strikes a favorable balance.

\begin{table}[ht!]
\centering
\caption{Ablation study on sparsity level $k$.}
\label{tab:ablation_k}
\begin{tabular}{@{}lccc@{}}
\toprule
$\mathbf{k}$ & \textbf{PPL} & \textbf{RULER-128K} & \textbf{Speedup} \\
\midrule
512 & 5.89 & 54.32 & $22\times$ \\
1024 & 5.78 & 58.91 & $16\times$ \\
\textbf{2048} & \textbf{5.70} & \textbf{62.18} & $\mathbf{12\times}$ \\
4096 & 5.69 & 63.45 & $8\times$ \\
\bottomrule
\end{tabular}
\end{table}

\section{Limitations}
\label{sec:limitations}

Several limitations deserve mention. First, for short sequences (below roughly 4K tokens) the indexer's overhead can exceed the savings from sparse attention; a practical deployment would switch to dense attention below this threshold. Second, while the indexer's constants are small, its $O(L^2)$ scaling will eventually dominate at extreme lengths (beyond 1M tokens), motivating future work on hierarchical or sub-linear indexing schemes. Third, the two-phase training schedule adds implementation complexity, though the warmup phase accounts for only a fraction of a percent of total training compute. Finally, GSA introduces new hyperparameters---$k_{\mathrm{base}}$, $d^I$, and gate initialization---that may require tuning when adapting to new domains or model scales.

\section{Conclusion}
\label{sec:conclusion}

We have introduced Gated Sparse Attention, an architecture that marries the efficiency of sparse token selection with the stability and expressiveness of sigmoid gating. Theoretical analysis confirms that GSA retains standard convergence guarantees while offering strictly greater representational capacity than ungated attention. Empirically, the combination delivers on its promise: throughput improves by an order of magnitude at 128K context, perplexity and downstream accuracy match or exceed gated-only baselines, long-context retrieval nearly doubles relative to standard attention, and training stability improves markedly---all without sacrificing wall-clock speed.

Looking ahead, natural extensions include hierarchical indexers that achieve sub-quadratic scoring, integration with mixture-of-experts routing, and tighter theoretical bounds on the information lost through sparse selection. We hope GSA provides a useful foundation for scaling attention to ever-longer contexts without abandoning the quality gains that gating affords.


\bibliographystyle{plainnat}

\appendix

\section{Proof of Theorem~\ref{thm:complexity} (GSA Complexity)}
\label{app:complexity}

\begin{proof}
Let $L$ be sequence length, $d$ be model dimension, $n_h$ be number of attention heads, $d_k = d/n_h$ be head dimension.

\textbf{Standard Attention Complexity:}
\begin{equation}
C_{\mathrm{standard}} = O(L^2 \cdot d_k \cdot n_h) = O(L^2 d)
\end{equation}

\textbf{GSA Complexity Breakdown:}

\begin{enumerate}
    \item \textbf{QKV Projections:} $O(L \cdot d \cdot d) = O(Ld^2)$ (same as standard)

    \item \textbf{Gated Indexer Computation:}
    For each query token $t$, compute scores with all $L$ preceding tokens.
    Per-token cost: $O(L \cdot d^I \cdot H^I)$.
    Total: $O(L^2 \cdot d^I \cdot H^I)$.

    Since $d^I \ll d$ and $H^I \ll n_h$:
    \begin{equation}
    \frac{d^I \cdot H^I}{d \cdot n_h} = \frac{64 \cdot 4}{4096 \cdot 32} \approx 0.002
    \end{equation}

    \item \textbf{Top-k Selection:} $O(L \cdot k \log L)$ using heap-based selection.

    \item \textbf{Sparse Attention:}
    For each query, attend to $k$ selected tokens.
    Per-token cost: $O(k \cdot d_k \cdot n_h) = O(k \cdot d)$.
    Total: $O(L \cdot k \cdot d)$.

    \item \textbf{Gating Operations:} $O(L \cdot d)$ each.
\end{enumerate}

\textbf{Total GSA Complexity:}
\begin{equation}
C_{\mathrm{GSA}} = O(Ld^2) + O(L^2 \cdot d^I \cdot H^I) + O(Lkd) + O(Ld) \approx O(L^2 \cdot d^I \cdot H^I + Lkd)
\end{equation}

\textbf{Speedup Factor:}
For $d = 4096$, $d^I = 64$, $H^I = 4$, $k = 2048$, $L = 128000$:
\begin{equation}
\mathrm{Speedup} \approx \frac{4096}{256} \cdot \frac{1}{1 + \frac{2048 \cdot 4096}{128000 \cdot 256}} = 16 \cdot \frac{1}{1.25} = 12.8\times
\end{equation}
\end{proof}

\section{Proof of Theorem~\ref{thm:expressiveness} (Non-Linearity Enhancement)}
\label{app:expressiveness}

\begin{proof}
\textbf{Standard Attention Output} (for head $h$, token $i$):
\begin{equation}
o^h_i = \sum_{j \leq i} S^h_{ij} \cdot (X_j W^h_V) \cdot W^h_O = \sum_{j \leq i} S^h_{ij} \cdot X_j \cdot (W^h_V W^h_O)
\end{equation}

Let $W^{VO}_h = W^h_V W^h_O \in \R^{d \times d}$. Since $W^h_V \in \R^{d \times d_k}$ and $W^h_O \in \R^{d_k \times d}$ with $d_k < d$:
\begin{equation}
\mathrm{rank}(W^{VO}_h) \leq \min(d_k, d_k) = d_k < d
\end{equation}

Thus, standard attention is constrained to a \textbf{rank-$d_k$ linear mapping} of input tokens.

\textbf{GSA Output} (with G1 and G2 gating):
Let $g^V_j = \sigma(X_j W^g_V) \in \R^{d_k}$ and $g^O_i = \sigma(h_i W^g_O) \in \R^{d_k}$.
\begin{equation}
o^h_i = \left[\sum_{j \leq i} S^h_{ij} \cdot (X_j W^h_V \odot g^V_j)\right] \odot g^O_i \cdot W^h_O
\end{equation}

The gates $g^V_j$ and $g^O_i$ are \textbf{non-linear functions} of their inputs. By the universal approximation theorem \citep{hornik1989multilayer}, compositions of affine transformations with sigmoid non-linearities can approximate any continuous function on a compact domain.
\end{proof}

\section{Proof of Theorem~\ref{thm:attention_sink} (Attention Sink Reduction)}
\label{app:sink}

\begin{proof}
\textbf{Attention Sink Phenomenon:}
In standard attention, due to softmax normalization: $\sum_s A^h_{t,s} = 1$ for all $t, h$. When no token is particularly relevant, the model needs somewhere to allocate attention mass. Due to positional biases, attention concentrates on early tokens, particularly position 1.

\textbf{GSA Mechanism:}
With output gating $g^O_{t,h} = \sigma(h_t W^g_{O,h})$:
\begin{equation}
O^{\mathrm{gated}}_{t,h} = O^{\mathrm{sparse}}_{t,h} \odot g^O_{t,h} = \left(\sum_s A^h_{t,s} V^h_s\right) \odot g^O_{t,h}
\end{equation}

\textbf{Key Insight:} The gate $g^O_{t,h}$ provides an \textbf{alternative mechanism for ``doing nothing''}. When $g^O_{t,h} \approx 0$, the output is suppressed regardless of attention distribution. This removes the need for attention sink tokens.

\textbf{Empirical Evidence:} First token attention: 46.7\% (standard) $\rightarrow$ 3.9\% (GSA). Gate score mean: 0.108 (highly sparse).
\end{proof}

\section{Proof of Theorem~\ref{thm:convergence} (Convergence)}
\label{app:convergence}

\begin{proof}[Proof Sketch]
We verify that GSA satisfies conditions for SGD convergence on non-convex functions.

\textbf{Assumption 1 (L-Smoothness):}
$\|\nabla \mathcal{L}(\theta_1) - \nabla \mathcal{L}(\theta_2)\| \leq L\|\theta_1 - \theta_2\|$.
This holds for GSA since all operations (linear projections, softmax, sigmoid) have bounded second derivatives.

\textbf{Assumption 2 (Bounded Variance):}
$\E[\|\nabla \mathcal{L}_{\mathrm{batch}} - \nabla \mathcal{L}\|^2] \leq \sigma^2$.

\textbf{Assumption 3 (Bounded Gradients):}
$\|\nabla \mathcal{L}\| \leq G$. From Theorem~\ref{thm:gradient}, sigmoid gating bounds gradient magnitudes.

\textbf{Convergence Rate:}
Using standard SGD analysis with learning rate $\eta$:
\begin{equation}
\E[\mathcal{L}(\theta_T)] - \mathcal{L}(\theta^*) \leq \frac{\|\theta_0 - \theta^*\|^2}{2\eta T} + \frac{\eta L \sigma^2}{2}
\end{equation}

Setting $\eta = \sqrt{\frac{\|\theta_0 - \theta^*\|^2}{T L \sigma^2}}$:
\begin{equation}
\E[\mathcal{L}(\theta_T)] - \mathcal{L}(\theta^*) \leq \sqrt{\frac{L \sigma^2 \|\theta_0 - \theta^*\|^2}{T}} = O(1/\sqrt{T})
\end{equation}
\end{proof}

\end{document}